\documentclass{article}
\pdfoutput=1
\usepackage{ijcai16}
\usepackage{times}
\usepackage[ruled]{algorithm2e}
\usepackage{algpseudocode}
\usepackage{amsthm}
\usepackage{amssymb}
\usepackage{amsmath}
\usepackage{graphicx}
\usepackage{enumitem}
\usepackage{subcaption}
\usepackage{wrapfig}
\usepackage{pbox}

\newtheorem{theorem}{Theorem}

\DeclareMathOperator*{\argmax}{arg\,max}

\title{Latent Contextual Bandits and their Application to Personalized
  Recommendations for New Users}
\author{Li Zhou \and Emma Brunskill \\
Computer Science Department \\ Carnegie Mellon University \\
\{lizhou, ebrun\}@cs.cmu.edu}

\begin{document}
\maketitle

\begin{abstract}
  Personalized recommendations for new users, also known as the cold-start
  problem, can be formulated as a contextual bandit problem. Existing contextual
  bandit algorithms generally rely on features alone to capture user
  variability. Such methods are inefficient in learning new users' interests. In
  this paper we propose Latent Contextual Bandits. We consider both the benefit
  of leveraging a set of learned latent user classes for new users, and how we
  can learn such latent classes from prior users. We show that our approach
  achieves a better regret bound than existing algorithms. We also demonstrate
  the benefit of our approach using a large real world dataset and a preliminary
  user study.
\end{abstract}

\section{Introduction}
In general we desire recommender systems that can quickly start providing good
recommendations for \emph{new users}, which is particularly challenging as no
prior information for new users is available. This is often known as the
\emph{cold-start} problem. Despite the lack of prior information for new users,
such systems typically have interacted with millions of previous
users. Therefore, this problem can be cast as an instance of lifelong learning
across sequential decision making tasks: how should information from prior users
be leveraged to help improve the recommendations for a new user? Standard
techniques like collaborative filtering~\cite{koren2009matrix} provide good
answers to this challenge, but such approaches are typically limited to
myopically providing a single recommendation, rather than reasoning about the
multi-step interactions the system may have with the user. This is important,
because across a sequence of interactions it may be useful for the system to
actively gather information (potentially sacrifice immediate performance
outcomes) in order to maximize its benefit over the longer run with the
individual in question.

One approach to this is to model users by a contextual bandit
model~\cite{bubeck2012regret,DBLP:journals/corr/Zhou15c} with a single shared
set of model parameters, and all prior users' data can be leveraged to fit those
model parameters for use in interacting with a new user. Example algorithms are
LinUCB~\cite{li2010contextual}, Thompson sampling with linear
payoffs~\cite{agrawal2013thompson}, and
CofineUCB~\cite{yue2012hierarchical}. However, these approaches work well only
when there are many available features that describe users and capture user
variability. If those features are not available, then we may need to fall back
on a population average that may make poor recommendations for the current new
individual.

At the other extreme is to use learning algorithms such as LinUCB and Thompson
sampling with linear payoffs to learn from scratch for each new user
separately. Such systems can provide full personalization to an individual
(using model parameters learned only for that user), but may take an enormous
amount of interactions to achieve this, yielding very little value for a long
period (and potentially causing the user to get frustrated or cease using the
system).

We instead propose an approach that provides \textit{partial
  personalization}. We assume that users can be described as each belonging to
one of a finite set of latent classes. Each class may be associated with a
different set of model parameters, but within a class all individuals share the
same parameters. Compared with the two extreme approaches mentioned above,
partial personalization does not fully rely on user features to capture user
variability, instead it leverages users' latent class structure to more quickly
start providing good recommendations for new users.

Latent class structure has been explored in the non-contextual Multi-armed
bandits setting \cite{lazaric2013sequential,maillard2014latent}. In the
contextual setting, the most closely related work is
CLUB~\cite{gentile2014online}. CLUB learns an underlying graph structure of
users based on user similarities and serves a group of users by taking advantage
of the learned graph structure. However, as we will show later, our algorithm is
theoretically and empirically better than CLUB.

In this paper, we focus on the latent contextual bandit setting. We consider
both the benefit of using a set of learned latent classes, and how we can obtain
such latent classes from a prior set of data or online. We also provide a
formal analysis of the regret in this setting, by building on recent progress on
latent variable learning of regression model mixtures using tensor
methods~\cite{tejasvi2013spectral}, to bound the performance obtained by
learning and leveraging a set of latent models learned from data. We then
demonstrate the benefit of our approach in simulation and an unbiased offline
evaluation using a large real world dataset, as well as a preliminary user
study. Our results suggest a substantial benefit of our latent contextual bandit
approach.

\section{Our Approach}
\subsection{Problem Formulation}
\label{sec:problem_formulation}
We assume there is a sequence of contextual bandit tasks, and each contextual
bandit task involves multi-step interactions with a particular
user.\footnote{While in the paper we assume users come sequentially and interact
  with our algorithm, our approach can also handle interleaved users. As we will
  describe later, in phase 1 of our algorithm, it doesn't matter if users
  interleave as LinUCB is used. In phase 2, we can fix the current set of
  clusters and parameters for a new user when he/she first arrive, and use that
  for the entire time we interact with that user.} Let there be $\tilde{U}$
users, where $\tilde{U}$ may be infinite. Each user $u$ belongs to one of a
finite set of $N$ latent classes. Denote by $c_u$ the (unknown) latent class of
user $u$. Users within the same latent class share similar interests and
behaviors. Each task (series of interactions with a single user) is assumed to
last for $T_u$ steps.

During a task, in each time step $t_u \in \{1,\ldots,T_u\}$, the algorithm
observes both the current user $u$ and a set $\mathcal{A}_{t_u}$ of arms (items)
together with their d-dimensional feature vectors $x_{t_u, a} \in \mathbb{R}^d$
for all $a \in \mathcal{A}_{t_u}$. $||x_{t_u, a}||_2 \le 1$. The feature vector
$x_{t_u, a}$ captures the information of both user $u$ and arm $a$ at time
$t_u$. For example, $x_{t_u, a}$ could be the linear concatenation of user and
arm feature vectors. We assume the size of $\mathcal{A}_{t_u}$ is fixed:
$|\mathcal{A}_{t_u}| = K$. These $K$ feature vectors are together referred to as
the \emph{context} $C_{t_u}$ at $t_u$: $C_{t_u} = \{x_{t_u, a} | a \in
\mathcal{A}_{t_u}\}$.

The algorithm then recommends an arm $a_{t_u}$ to the current user $u$, and
receives reward $r_{t_u} \in [0, 1]$ from the user. We assume that the reward is
a noisy linear function of the current user's latent class. More precisely, each
latent class $h \in \{1, ..., N\}$ is associated with an (unknown) weight vector
$\beta_h \in \mathbb{R}^d$. The reward of an arm $a \in \mathcal{A}_{t_u}$ is
given by
\begin{align*}
r_{t_u, a} = \beta_{c_u}^\top x_{t_u, a} + \epsilon_{c_u}
\end{align*}
where $\epsilon_{c_u}$ follows a Gaussian distribution with zero mean and
bounded variance. Let $a^*_{t_u} = \argmax_{a\in \mathcal{A}_{t_u}}
\beta^\top_{c_u} x_{t_u, a}$ be the arm with highest expected reward at time
$t_u$, and $a_{t_u}$ be the arm selected at $t_u$. Then the algorithm's goal is
to minimize the \emph{regret}, which is defined as
%\small
%\begin{align*}
%  \text{Reg}(\tilde{U}, T_{u:u \in [\tilde{U}]}) = \sum_{u=1}^{\tilde{U}}
%  \sum_{t_u=1}^{T_u} \beta_{c_u}^\top x_{t_u, a_{t_u}^*} -
%  \sum_{u=1}^{\tilde{U}} \sum_{t_u=1}^{T_u} \beta_{c_u}^\top x_{t_u, a_{t_u}}
%\end{align*}
%\normalsize
\begin{align*}
  \text{Reg}(\tilde{U}, T_{u:u \in \{1, ..., \tilde{U}\}})&  \\ =  
  \sum_{u=1}^{\tilde{U}}
  \sum_{t_u=1}^{T_u} & \beta_{c_u}^\top x_{t_u, a_{t_u}^*} -
  \sum_{u=1}^{\tilde{U}} \sum_{t_u=1}^{T_u} \beta_{c_u}^\top x_{t_u, a_{t_u}}
\end{align*}
\subsection{Latent Contextual Bandits}
\label{sec:lccb}
Our algorithm, Latent Contextual Bandits (LCB), is described in Algorithm
\ref{algo:lb}. LCB learns the set of latent models from prior users and
leverages the learned models to make recommendations for new users. The
algorithm consists of two phases. In phase 1, LCB simply runs LinUCB algorithm
on the first $J$ users and collects the pulled arms and rewards. (line $2-11$,
Algorithm \ref{algo:lb}). The reason to do phase 1 is that initially when LCB
starts from scratch, there are no prior users or training data for LCB to learn
the latent models. Therefore, phase 1 is the bootstrap phase of the algorithm. A
short phase 1 will cause a high model estimation error at the early stage of
phase 2, while a long phase 1 will cause large regret in phase 1. In Section
\ref{sec:theory} we will discuss how to pick the length of phase 1 to get low
overall regret bound. In real world systems, usually we already have a huge set
of interactions made by prior users, then phase 1 is not needed.

In phase 2, LCB train/re-train latent models using data collected in both phases
1 and 2 so far (line $13$, Algorithm \ref{algo:lb}). We will show how to learn
the latent models in Section \ref{sec:mlr}. In practice, we may want to re-train
the latent models after a batch set of users instead of each user.

Meanwhile, in phase 2 LCB should leverage the learned latent models to improve
performance for new tasks (users). Though there are many ways to do this, we
propose an approach that first constructs a policy for each learned latent
model, and then uses a contextual bandit algorithm that can adaptively select
across the policies for a new task. A policy is a function that takes a context
as input and returns an arm or a distribution over arms. For example, one policy
could be a function that always return the arm with the highest expected reward
estimated by a learned latent model. There already exist numerous contextual
bandit algorithms that take as input a finite set of policies and compete with
the best policy inside the policy
set~\cite{beygelzimer2011contextual,agarwal2014taming}, so LCB can build upon
these existing works. However, LCB in phase 2 offers multiple advantages
relative to these prior works: the policy set $N$ is often smaller than the set
of policies considered by generic contextual bandit approaches; LCB
automatically constructs the set of policies (instead of requiring an oracle or
expert to provide a good set); and assuming the problem setting holds, the set
of $N$ policies is sufficient to enable optimal performance for any new task, in
contrast to standard contextual bandit approaches which can only achieve
performance as good as the input policies (which may not achieve optimal
performance).

More precisely, LCB constructs one policy for each learned latent model (line
$14$, Algorithm \ref{algo:lb}), and then runs a pre-selected contextual bandit
algorithm $\mathcal{B}$ that takes in the set of $N$ learned policies for the
$N$ latent contextual bandit tasks (line $15-22$, Algorithm \ref{algo:lb}). We
will discuss specific choices of $\mathcal{B}$ and ways to construct policies in
Section \ref{sec:bandits}, and we will shortly provide a theoretical analysis of
our approach in Section \ref{sec:theory}.
 
\IncMargin{0.5em}
\begin{algorithm}[!t]
\SetAlgoNoLine
\LinesNumbered
\SetAlgoNlRelativeSize{0}
\SetNlSty{}{}{:}
\SetAlCapHSkip{0em}
\SetCommentSty{emph}
\DontPrintSemicolon
\caption{Latent Contextual Bandits (LCB)}
\label{algo:lb}
\Indm
  \KwIn{$J\in \mathbb{R}^+$: number of users in phase 1 \newline
        $N \in \mathbb{R}^+$: number of latent classes \newline
        $\mathcal{P}$: policies construction strategy \newline
        $\mathcal{B}$: contextual bandit algorithm}

\BlankLine
\Indp
  Samples $\mathcal{D} = \varnothing$ \;
  \BlankLine
  \tcp{Phase $1$}
  Create a LinUCB instance \;
  \For{user $u \in \{1, 2, ..., J\}$} {
    \For{$t_u \in \{1,2, ..., T_u\}$} {
      Observe context $C_{t_u} = \{x_{t_u,a} \in \mathbb{R}^d | a \in
      \mathcal{A}_{t_u}\}$ \;
      Pull $a_{t_u} = \mathrm{LinUCB}(C_{t_u})$  \;
      Observe reward $r_{t_u}$\;
      Update LinUCB based on reward \;
      Add $(x_{t_u, a_{t_u}}, r_{t_u})$ to $\mathcal{D}$ \;
    }
  }
  
  \BlankLine 
  \tcp{Phase $2$}
  \For{user $u \in \{J+1, J+2, ...\}$}{
    Learn $N$ latent models $\{\hat{\beta}_1, ..., \hat{\beta}_N\}$ using data
    $\mathcal{D}$ \;
    Construct $N$ policies $\{\mathcal{P}(\hat{\beta}_1, \ \cdot\ ), ...,
    \mathcal{P}(\hat{\beta}_N,\ \cdot\ )\}$  \;
    Create a $\mathcal{B}$ instance for $u$\;
      \For{$t_u \in \{1, 2, ...T_u\}$}{
      Observe context $C_{t_u} = \{x_{t_u,a} \in \mathbb{R}^d | a \in
      \mathcal{A}_{t_u}\}$ \;
      %Pull $a_{t_u} = \mathcal{B}(C_{t_u}, \mathcal{P}(\hat{\beta}_1), ..., \mathcal{P}(\hat{\beta}_N))$ \;
      Pull
      %\small
      %$a_{t_u} = \mathcal{B}_u(C_{t_u}, \{\mathcal{P}(\hat{\beta}_h,\ \cdot\ ):h\in \{1,...,N\}\})$ \;
      $a_{t_u} = \mathcal{B}_u(C_{t_u}, \{\mathcal{P}(\hat{\beta}_1,\cdot\ ), ..., \mathcal{P}(\hat{\beta}_N,\cdot\ )\})$ \;
      %\normalsize
      Observe reward $r_{t_u}$\;
      Update $\mathcal{B}_u$ based on reward\;
      Add $(x_{t_u, a_{t_u}}, r_{t_u})$ to $\mathcal{D}$ \;
      }
}
\end{algorithm}

\subsection{Learn Latent Models from Past Users}
\label{sec:mlr}
We model latent user classes using a mixture of linear
regressions~\cite{viele2002modeling}. A mixture of linear regressions consists
of $N$ mixture components, each is a linear regression model. Let
$\Theta=\{\pi_h, \beta_h, \sigma^2_h | h \in \{1,...,N\}\}$ be the model
parameters, where $\pi_h$ is the mixture proportion, $\beta_h$ is the
coefficient vector, and $\sigma^2_h$ is the variance of the response. Then the
likelihood of mixture of linear regressions is defined as
\small
\begin{align*}
\mathrm{L}(\Theta; \mathcal{D}) = \prod_{(r, x) \in \mathcal{D}}\left( \sum_{h=1}^N
  \pi_h \mathcal{N}\left(r|\beta_h^\top x, \sigma_h^2\right) \right)
\end{align*}
\normalsize
where $\mathcal{N}(r|\mu,\sigma^2)$ is the probability density function of a
Gaussian distribution with mean $\mu$ and variance $\sigma^2$. One classic
algorithm to learn a mixture model is the Expectation-Maximization (EM)
algorithm. However, EM does not guarantee convergence to the globally optimal
parameters, and it does not provide finite sample guarantees on the quality of
the resulting parameter estimates. On the other hand, tensor decomposition based
methods, as we will describe shortly, give us finite sample guarantees which can
be further used to derived our regret bound.
\subsubsection{Learn Latent Models using Spectral Experts}
Anandkumar et al.~\shortcite{anandkumar2014tensor} showed that tensor
decomposition can efficiently recover parameters for a wide class of latent
variable models. They exploited a special tensor structure derived from second
and third-order moments of the observations, and apply the robust tensor power
method to recover model parameters. Spectral Experts~\cite{tejasvi2013spectral},
built on top of Anandkumar et al.'s work, provide provably consistent estimator
for mixture of linear regressions. Our algorithm uses Spectral Experts to
estimate parameters of mixture of linear regressions. Later in Section
\ref{sec:theory} we also use the parameter error bound provided by Spectral
Experts to bound the regret of our algorithm.

Though Spectral Experts algorithm has appealing theoretical properties, it is
not particularly sample efficient and it is computationally
expensive. Therefore, in the following section, we also derive and implement a
computationally efficient Gibbs sampling based procedure to estimate parameters
of mixtures of linear regressions.

\subsubsection{Learn Latent Models using Gibbs Sampling}
Gibbs sampling is an efficient inference technique to learn
latent models for large scale dataset. We derive a sampling procedure for
Dirichlet Process~\cite{neal2000markov} mixtures of linear regressions. By using
a Dirichlet Process prior, we do not need to specify the number of latent
models. Specifically, we assume the prior of $\beta_h$ and $\sigma_h^2$ follow
the Normal-inverse-Gamma distribution and the prior of $\pi$ follows GEM
distribution~\cite{murphy2012machinegem}, which is used by the stick-breaking
construction of the Dirichlet process. The generative process is as follows:
\begin{enumerate}[leftmargin=0.5cm]
\item $\alpha \sim \text{Gamma}(u_0, v_0)$
\item $\pi \sim \text{GEM}(\alpha)$ 
\item For each latent model $h \in \{1, ..., N\}$
\begin{enumerate}
\item $(\beta_h, \sigma_h^2) \sim \text{NIG}(w_0, V_0, a_0, b_0)$
\end{enumerate}
\item For each user $u \in \{1, ..., U\}$
\begin{enumerate}
\item $c_u \sim \text{Categorical}(\pi)$
\item For each interaction $t_u \in \{1, ..., T_u\}$
\begin{enumerate}
\item $r_{t_u} \sim \mathcal{N}(\beta_{c_u}^\top x_{{t_u}, a_{t_u}},
  \sigma_{c_u}^2)$
\end{enumerate}
\end{enumerate}
\end{enumerate}
We use collapsed Gibbs sampling to sample $c_u$ and $\alpha$. Denote all the
rewards of a user $u$ by $r_u = \{r_{t_u}:t_u \in \{1, ..., T_u\}\}$. To sample
$c_u$,
\begin{align}
&P(c_u=h | c_{-u}, r_u, \alpha, w_0, V_0, a_0, b_0)  \notag \\
&\propto P(c_u = h | c_{-u}, \alpha)P(r_u|r_{-u}^h, w_0, V_0, a_0, b_0) 
  \label{eq:sample_z}
\end{align}
where $c_{-u} = \{c_{u'} : u' \ne u\}$ and  $r_{-u}^h = \{r_{u'}: c_{u'}=h, u'
\ne u\}$. The first term in Equation (\ref{eq:sample_z}) is given by the Chinese
Restaurant Process (CRP) \cite{neal2000markov}, the second term in Equation
(\ref{eq:sample_z}) is the posterior predictive distribution of $r_u$ given
$r_{-u}^h$, and it follows Multivariate t-distribution
\cite{murphy2012machine}. To sample $\alpha$, we adopt the auxiliary variable
method \cite{escobar1995bayesian}.

\subsection{Leverage Learned Models for New Users}
\label{sec:bandits}
Let $\{\hat{\beta}_1, ..., \hat{\beta}_N\}$ be the $N$ learned latent models. We
define $N$ policies based on these models. There are two types of policies we
can define, one is deterministic, and the other one is probabilistic. The
deterministic one maps a context $C_{t_u}$ to an arm $a \in \mathcal{A}_{t_u}$:
\begin{align*}
\mathcal{P}(\hat{\beta}_h, C_{t_u}) = \argmax_{a \in \mathcal{A}_{t_u}}
  \hat{\beta}_h^\top x_{{t_u}, a}
\end{align*}
The probabilistic one maps a context $C_{t_u}$ to a categorical distribution
over arms:
\begin{align*}
\mathcal{P}(\hat{\beta}_h, C_{t_u}) = [p_{1}, p_{2}, ..., p_{|\mathcal{A}_{t_u}|}]
\end{align*}
where
\begin{align*}
p_{a} = \frac{\exp (\hat{\beta}_h^\top x_{{t_u}, a})}{\sum_{a\in \mathcal{A}_u}
  \exp (\hat{\beta}_h^\top x_{{t_u}, a})}
\end{align*}

The constructed polices can be used by many contextual bandit algorithms to
serve new users. If the policies are deterministic, possible contextual bandit
algorithms include Epoch-Greedy~\cite{NIPS2007_3178},
ILOVETOCONBANDITS~\cite{agarwal2014taming}, and Generalized Thompson
Sampling~\cite{DBLP:journals/corr/Li13e}. If the policies are probabilistic,
possible contextual bandit algorithms include EXP4~\cite{auer2002nonstochastic}
and EXP4.P~\cite{beygelzimer2011contextual}. The algorithm choice depends on the
desired outcome, and we will shortly consider specific choices for both our
theoretical analysis and empirical results.

\section{Theoretical Analysis}
\label{sec:theory}
In this section, we analyze LCB's expected regret. We assume the latent models
are learned using the Spectral Experts algorithm. Let $J$ be the number of users
in phase 1, $U$ be the number of users in phase 2, and $\tilde{U}=J+U$ be the
total number of users. Let $S = \sum_{u=1}^{J} T_u$, $T =
\sum_{u={J+1}}^{\tilde{U}}T_u$, and $\tilde{T} = S+T$ be the total number of
interactions. We denote the first $n$ positive integers by $[n]$. For
convenience, we define the \emph{true policy} of a user $u \in [\tilde{U}]$ as
the deterministic policy constructed by $\beta_{c_u}$ (the true latent model the
user belongs to).

For theoretical analysis, we make two minor changes to the Algorithm
\ref{algo:lb}. First, instead of running a single LinUCB instance for all users
in phase 1, we run a separate LinUCB instance for each user. The reason is that
under our realizability assumption (each user belongs to one of the latent
models) single LinUCB instance which runs for all users has linear regret
$O(S)$. Second, we collect $\tau_u$ i.i.d.\ samples from each user $u \in
[\tilde{U}]$, that is, we select arms uniformly at random for the first $\tau_u$
interactions for each user $u \in [\tilde{U}]$. When training latent models
using Spectral Experts, we only use these i.i.d.\ samples. We do this because
Spectral Experts requires i.i.d.\ training examples to get theoretical guarantee
on the parameter error bound.

Assume $T_u \le L$ for all $u \in [\tilde{U}]$ for some constant $L$. Denote the
minimum Euclidean distance of any two latent models by $\bigtriangleup$, that
is, $||\beta_h - \beta_{h'}|| \ge \bigtriangleup$ for any $h \ne h'$. The
following two theorems show a problem-independent expected regret bound which is
independent of $\bigtriangleup$ and a problem-dependent expected regret bound
which depends on $\bigtriangleup$.
\begin{theorem}
  Set $\tau_u = \sqrt{T_u}$ for all $u \in [\tilde{U}]$ and $J=\sqrt{L}$. Assume
  $\mathcal{P}$ constructs deterministic policies. If $\mathcal{B}$ is a
  contextual bandits algorithm with optimal regret bound (e.g.\ EXP4.P), then
  the problem-independent expected regret bound of LCB with respect to the true
  policy is
\small
\begin{align*}
\mathbb{E}\left[\mathrm{Reg}_{\mathrm{LCB}}\left(\tilde{U}, T_{u:u \in [\tilde{U}]}\right)\right] &= O\left(\sqrt{JS} +
                                                      d\sqrt{JS\ln (1+S)}\right)  \\
&+ O\left(3\sqrt{UT} + \sqrt{UTK\ln N} \right) \\
&= O\left(\sqrt{UTK\ln N} \right)
\end{align*}
\normalsize
as $UT \gg \max\left\{1, \frac{d^2}{K}\right\}JS$, that is, as $T$ and $U$ grows
large. Similarly, if $\mathcal{B}$ is EXP3~\cite{auer2002nonstochastic} which
treats each learned policy as an arm, then the problem-independent expected
regret bound of LCB with respect to the true policy is
\begin{align*}
\mathbb{E}\left[\mathrm{Reg}_{\mathrm{LCB}}\left(\tilde{U}, T_{u:u \in [\tilde{U}]}\right)\right] = O\left(\sqrt{UTN\ln N} \right)
\end{align*}
as $UT \gg \max\left\{1, \frac{d^2}{N}\right\}JS$, that is, as $T$ and $U$ grows large.
\end{theorem}
\begin{theorem}
  Set $\tau_u = 3$ for all $u \in [\tilde{U}]$ and $J=L^2$. Assume $\mathcal{P}$
  constructs deterministic policies and $\mathcal{B}$ is Epoch-Greedy, then the
  problem-dependent expected regret bound of LCB with respect to the true policy
  is 
\begin{align*}
&\mathbb{E}\left[\mathrm{Reg}_{\mathrm{LCB}}\left(\tilde{U}, T_{u:u \in [\tilde{U}]}\right)\right] = \\
&O\left(3L^2+dL\sqrt{S\ln(L+1)} + \frac{UK}{\bigtriangleup^2} \left(\ln N + \ln (T+1)\right)\right)
\end{align*}
\end{theorem}
\begin{proof}[\textbf{Proof} {(Theorem 1)}]
The expected regret of LinUCB~\cite{DBLP:journals/corr/Zhou15c} is
\begin{align*}
\mathbb{E}\left[\mathrm{Reg}_{\mathrm{LinUCB}}(T_u)\right] = O(d\sqrt{T_u\ln(1+T_u)})
\end{align*}
so in phase 1 the expected regret of LCB is
\small
\begin{align}
&\mathbb{E}\left[\mathrm{Reg}_{\mathrm{LCB}}^{phase\_1}\left(J, T_{u:u \in [J]}\right)\right] \nonumber \\
&= O\left(\sum_{u=1}^J \left( \sqrt{T_u} + d\sqrt{(T_u-\sqrt{T_u})\ln(1+T_u-\sqrt{T_u})} \right)\right) \nonumber \\
&= O(\sqrt{JS} + d\sqrt{JS\ln(1+S)}) \label{eq:reg_phase1_proof}
\end{align}
\normalsize
where Equation (\ref{eq:reg_phase1_proof}) 
follows from the Cauchy-Schwarz inequality. 

We next need to bound the regret in phase 2. For each user $u \in \{J+1,..., J+U\}$ in
phase 2, define
\begin{align}
\hat{\beta}_u^* = \argmax_{\hat{\beta} \in \{\hat{\beta}_1, ..., \hat{\beta}_N\}}\sum_{t_u=1}^{T_u}
\mathrm{E}[r_{t_u, \mathcal{P}(\hat{\beta}, C_{t_u})}] \label{eq:def_beta_hat}
\end{align}
as the best model of that user within all estimated models. Also recall that
$\beta_{c_u}$ is the true model of the user $u$ and $\hat{\beta}_{c_u}$ is the
estimate of $\beta_{c_u}$ returned by Spectral Experts. Let $a^*_{t_u},
\tilde{a}^*_{t_u}$ and $\hat{a}^*_{t_u}$ be the arm proposed by $\beta_{c_u},
\hat{\beta}_{c_u}$ and $\hat{\beta}^*_{u}$. $\hat{\beta}^*_u$ achieves the
highest expected cumulative reward based on its definition, so it achieves
higher expected cumulative reward than $\hat{\beta}_{c_u}$, so
\begin{align}
\sum_{t_u=1}^{T_u} x_{t_u, \hat{a}^*_{t_u}}^\top\beta_{c_u} \ge \sum_{t_u=1}^{T_u}
  x_{t_u, \tilde{a}^*_{t_u}}^\top\beta_{c_u} \label{eq:thm1_hat_tilde_gap}
\end{align}
Meanwhile, we can bound the gap between the expected cumulative reward achieved
by $\beta_{c_u}$ and by $\hat{\beta}_{c_u}$ as follows:
\begin{align}
&x_{t_u, a^*_{t_u}}^\top\beta_{c_u} - x_{t_u, \tilde{a}^*_{t_u}}^\top\beta_{c_u}
  \nonumber \\
&\le x_{t_u, a^*_{t_u}}^\top\beta_{c_u} - x_{t_u, \tilde{a}^*_{t_u}}^\top\beta_{c_u} +
  x_{t_u, \tilde{a}^*_{t_u}}^\top\hat{\beta}_{c_u} - x_{t_u, a^*_{t_u}}^\top\hat{\beta}_{c_u} \nonumber \\
&= (x_{t_u, a^*_{t_u}}^\top\beta_{c_u} - x_{t_u, a^*_{t_u}}^\top\hat{\beta}_{c_u}) +
  (x_{t_u, \tilde{a}^*_{t_u}}^\top\hat{\beta}_{c_u} - x_{t_u,
  \tilde{a}^*_{t_u}}^\top\beta_{c_u}) \nonumber \\
&\le 2||\beta_{c_u} - \hat{\beta}_{c_u}|| \label{eq:thm1_tilde_star_gap}
\end{align}
The last step uses the fact that $||x_{t_u, a}||_2 \le 1$. Using Equation
(\ref{eq:thm1_hat_tilde_gap}) and (\ref{eq:thm1_tilde_star_gap}) together we can
bound the gap between the expected cumulative reward achieved by $\beta_{c_u}$
and by $\hat{\beta}^*_{u}$:
\begin{align*}
\sum_{t_u=1}^{T_u}x_{t_u, a^*_{t_u}}^\top\beta_{c_u} - \sum_{t_u=1}^{T_u} x_{t_u,
  \hat{a}^*_{t_u}}^\top\beta_{c_u} \le 2T_u ||\beta_{c_u} - \hat{\beta}_{c_u}||
\end{align*}
Chaganty and Liang~\shortcite{tejasvi2013spectral} showed that $||\beta_{c_u} -
\hat{\beta}_{c_u}|| = O\left(\frac{1}{\sqrt{n}}\right)$ where $n$ is the number
of training examples. Now if $\mathcal{B}$ is a bandits algorithm with optimal
expected regret bound $O(\sqrt{T_u K\ln N})$, then in phase 2 the expected
regret of Latent Contextual Bandits is
\small
\begin{align}
&\mathbb{E}\left[\mathrm{Reg}_{\mathrm{LCB}}^{phase\_2}\left(U, T_{u:u \in
  \{J+1,..., 
  J+U\}}\right)\right] \nonumber \\
&= O\left(\sum_{u=J+1}^{J+U} \left( \sqrt{T_u} + \sqrt{T_uK\ln N} + 2T_u
  ||\beta_{c_u} - \hat{\beta}_{c_u}|| \right)\right) \nonumber \\
&= O\left(\sum_{u=J+1}^{J+U} \left( \sqrt{T_u} + \sqrt{T_uK\ln N} +
  \frac{2T_u}{\sqrt{\sqrt{T_u}(u-1)}} \right)\right) \nonumber \\
&= O\left(\sum_{u=J+1}^{J+U} \left( \sqrt{T_uK\ln N} +
  \frac{3\sqrt{T_u}}{\sqrt{(1+(u-J-1)/\sqrt{T_u})}} \right)\right) \label{eq:thm1_lb_phase2_2} \\
&= O(3\sqrt{UT} + \sqrt{UTK\ln N}) \label{eq:thm1_lb_phase2}
\end{align}
\normalsize
Equation (\ref{eq:thm1_lb_phase2_2}) follows from $J\ge\sqrt{T_u}$ for all
$u$. Equation (\ref{eq:thm1_lb_phase2}) follows from bounding the last term in
(\ref{eq:thm1_lb_phase2_2}) by $3\sqrt{T_u}$, and then applying Cauchy-Schwarz
inequality.\footnote{The last term actually decreases at the rate of
  $O(1/\sqrt{u})$ with respect to $u$, so our regret bound gets tighter as $u$
  increases in phase 2, but only by a constant factor.} Similarly, if
$\mathcal{B}$ is EXP3 which achieves a regret of $O(\sqrt{TN\ln N})$, then the
expected regret of LCB in phase 2 is $O(3\sqrt{UT}+\sqrt{UTN\ln N})$. Finally,
by adding the regret bound of phase 1 and phase 2, we prove the theorem.
\end{proof}
\begin{proof}[\textbf{Proof} {(Theorem 2)}]
  The proof of Theorem 1 shows that the regret in phase 1 is 
\begin{align*}
&\mathbb{E}\left[\mathrm{Reg}_{\mathrm{LCB}}^{phase\_1}\left(J, T_{u:u \in [J]}\right)\right] \nonumber \\
&= O\left(\sum_{u=1}^J \left( \tau_u + d\sqrt{T_u\ln(1+T_u)} \right)\right) \nonumber
\end{align*}
If $\mathcal{B}$ is Epoch-Greedy, then based on Epoch-Greedy's problem-dependent
bound we have
\small
\begin{align}
&\mathbb{E}\left[\mathrm{Reg}_{\mathrm{LCB}}^{phase\_2}\left(U, T_{u:u \in \{J+1,...,
  J+U\}}\right)\right] \nonumber \\
&= O\left(\sum_{u=J+1}^{J+U} \left( \tau_u + \frac{K\ln (N(T_u+1))}{\bigtriangleup^2} + 2T_u ||\beta_{c_u} - \hat{\beta}_{c_u}|| \right)\right) \nonumber \\
&= O\left(\sum_{u=J+1}^{J+U} \left( \tau_u + \frac{K\ln (N(T_u+1))}{\bigtriangleup^2} + \frac{2T_u}{\sqrt{u-1}} \right)\right) \nonumber 
\end{align}
\normalsize
Set $J = L^2$, then the last term is less than or equal to 1. Set $\tau = 3$,
then the regret of LCB, by adding the regret in phase 1 and phase 2, is
\begin{align*}
&\mathbb{E}\left[\mathrm{Reg}_{\mathrm{LCB}}\left(\tilde{U}, T_{u:u \in [\tilde{U}]}\right)\right] = \\
&O\left(3L^2+dL\sqrt{S\ln(L+1)} + \frac{UK}{\bigtriangleup^2} \left(\ln N + \ln (T+1)\right)\right)
\end{align*}
which proves the theorem.
\end{proof}
\begin{table*}[!hpt]
\centering
\small
\begin{tabular}{|l|l|l|l|}
\hline
Algorithm & Expected Regret & Regret with respect to & Regret Type \\ \hline
Population EXP4.P & $O(\sqrt{\tilde{T}K\ln N})$ & best policy in pre-defined policy set & problem-independent \\ \hline
Individual EXP4.P & $O(\sqrt{\tilde{U}\tilde{T}K\ln N})$  & best policy in
                                                            pre-defined policy
                                                            set &problem-independent \\ \hline
Population LinUCB~\footnotemark& $O(d\sqrt{\tilde{T}\ln(1+\tilde{T})})$ & best
                                                                          average
                                                                          policy
                                                                          of all
                                                                          users
                                            & problem-independent \\ \hline
Individual LinUCB & $O(d\sqrt{\tilde{U}\tilde{T}\ln(1+\tilde{T})})$ & true
                                                                      policy of
                                                                      each user
                                            & problem-independent \\ \hline
EXP4.P enum-policies & $O(\sqrt{\tilde{U}\tilde{T}KC\ln K})$ & true
                                                                        policy
                                                                        of each
                                                                        user &
                                                                        problem-independent \\ \hline
LCB (Theorem 1) & $O(\sqrt{UT\min\{K,N\}\ln N})$ & true policy of each user & problem-independent \\ \hline
CLUB & $O(
       N+\tilde{U}\sqrt{Nd}+\frac{\tilde{U}d}{\bigtriangleup^2}\ln (\tilde{T}+1)+dN\sqrt{\tilde{T}})$
                   & true policy of each user & problem-dependent\\ \hline
LCB (Theorem 2) & $O(3L^2+dL\sqrt{S} +
                  \frac{UK}{\bigtriangleup^2}(\ln N + \ln (T+1)) )$ & true policy of each user & problem-dependent
  \\ \hline
\end{tabular}
\caption{Expected regret bounds of LCB and baseline algorithms. We use $N$ to denote both
  the number of policies in the three EXP4.P variants and the number of latent models in
  LCB and CLUB. We use $C$ to denote the total number of contexts.}
\label{table:reg_compare}
\end{table*}

To put these results in context, we now compare our regret results to several
other approaches. We compare the following algorithms: 1) \textbf{Population
  EXP4.P}: runs a single EXP4.P model for all users, 2) \textbf{Individual
  EXP4.P}: runs a separate EXP4.P model for each user, 3) \textbf{Population
  LinUCB}: runs a single LinUCB model for all users, 4), \textbf{Individual
  LinUCB}: runs a separate LinUCB model for each user, 5) \textbf{EXP4.P
  enum-policies}: enumerates all policies by mapping all possible
contexts to all possible arms and then runs EXP4.P on each user (assuming
contexts are enumerable), 6) \textbf{CLUB}, 7) \textbf{LCB}. Table
\ref{table:reg_compare} shows the expected regret of each algorithm; it also
shows the policy each algorithm is competing with when deriving the regret
bound. Keep in mind that all comparisons are under our realizability assumption
that each user belongs to one of the latent models.

Within the 7 algorithms in Table \ref{table:reg_compare}, 3 of them does not
compete with the true policy of each user: Population LinUCB doesn't distinguish
between users from different classes, so it is competing with the best average
policy of all users; Population/Individual EXP4.P requires a set of pre-defined
policies as input, and compete with the best one inside the policy set instead
of the true policy of each user.

The remaining 4 algorithms all compete with the true policy of each user;
however, LCB achieves the best expected regret bound. The problem-independent
regret bound of Individual LinUCB is linear with respect to the contexts'
feature dimension which is often very large. Define $C$ as the total number of
contexts. EXP4.P enum-policies has a $C$ term in its
problem-independent regret bound which is often large or even infinite. LCB's
problem-independent regret bound, on the other hand, only has square root
dependence on $U, T$ and $\min\{K, N\}$.

Table \ref{table:reg_compare} also shows the problem-dependent regret bound of
CLUB analyzed by Gentile et al.~\shortcite{gentile2014online}. Both LCB and
CLUB's problem-dependent regret bounds depends on $1/\bigtriangleup^2$. However,
CLUB has a square root dependence on $\tilde{T}$, while LCB only has a square
root dependence on $S$ (number of interactions in phase 1) which is a
constant. Moreover, the analysis of CLUB (see Appendix of Gentile et
al.~\shortcite{gentile2014online}) shows that CLUB's expected regret on a new
user $u$ is linear with respect to $T_u$ when $T_u < B$ for some constant B that
is in the order of $O(\frac{d}{\bigtriangleup^2})$. $B$ may be enormous if
$\bigtriangleup^2$ is small. Meanwhile, LCB's problem-independent regret bound
guarantees that LCB's expected regret on a new user $u$ is always sublinear
(square root) with respect to $T_u$ even when $T_u$ is small.

\section{Experiments}
 In this section, we evaluate our algorithm both on simulation and on a large
 real world dataset from Yahoo!. We compare the following algorithms: 1)
 \textbf{LCB}: our approach. We choose Generalized Thompson Sampling as
 $\mathcal{B}$ in Algorithm \ref{algo:lb}. For simulation we use Spectral
 Experts to learn latent models, but due to time/memory constrains, for large
 real world dataset we use Gibbs sampling to learn latent models; 2)
 \textbf{LCB\_GT}: this is similar to LCB, except that instead of learning the
 latent models, we provide true latent models to the algorithm. 3)
 \textbf{CLUB}; 4) \textbf{Population LinUCB}: runs single LinUCB instance on
 all users; 5) \textbf{Individual LinUCB}: runs a separate LinUCB instance on
 each user; 6) \textbf{Random} : selects each arm uniformly at random.
\footnotetext{Under our realizability assumption, Population LinUCB has
  $O(\tilde{T})$ linear regret, so here we show the regret under its own realizability
assumption}
\subsection{Simulation}
We artificially created 5 latent models as shown in Figure
\ref{fig:sim_latent_models}. Each model had 10 parameters, and 4 of them were
assigned higher weights. Users were sampled uniformly at random from these
latent models. For each user interaction we generated 20 arms, that is,
$|\mathcal{A}_{t_u}|=20$. Each arm $a$ was associated with a feature vector
$x_{t_u, a}$ sampled uniformly from $[-1, 1]^{10}$, and was normalized so that
$||x_{t_u, a}||_2 = 1$. We sampled the reward of each arm $a$ from
$\mathcal{N}(\beta_{c_u}^\top x_{t_u, a}, \sigma^2)$ with $\sigma=0.1$. For LCB,
we set $J=50$, and in phase 2 we re-trained the latent models after every 50
users.
\begin{wrapfigure}{R}{0.25\textwidth}
  \centering
    \includegraphics[width=0.25\textwidth]{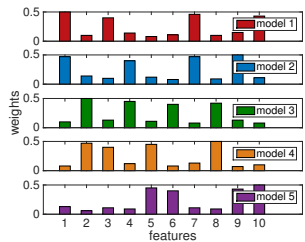}
  \caption{Models in simulation}
  \label{fig:sim_latent_models}
\end{wrapfigure}

In the first experiment, we fixed $T_u=20$ for all users, and reported the
averaged per-user regret vs.\ number of users. Results are shown in Figure
\ref{fig:sim_online}. We can see that the regret of LCB was about $35\%$ lower
than CLUB and $50\%$ lower than Population LinUCB. In the second experiment, we
fixed the number of users to $1000$, and varied $T_u$ from $10$ to $100$. $T_u$
were set to the same for all users. Figure \ref{fig:sim_num_steps} shows the
averaged per-user regret vs.\ $T_u$. We can see that LCB outperformed CLUB and
LinUCB, and as $T_u$ increased, the gap between their regret also increased.
Also, when $T_u$ was more than 40, Individual LinUCB started to learn a good model
for each user, and outperformed Population LinUCB, however, it still had much
higher regret than LCB.
\begin{figure}[h]
\centering
%\begin{minipage}[c]{0.48\textwidth}
\begin{minipage}[c]{1\linewidth}
\begin{subfigure}{.495\textwidth}
  \centering
  \includegraphics[width=1\linewidth]{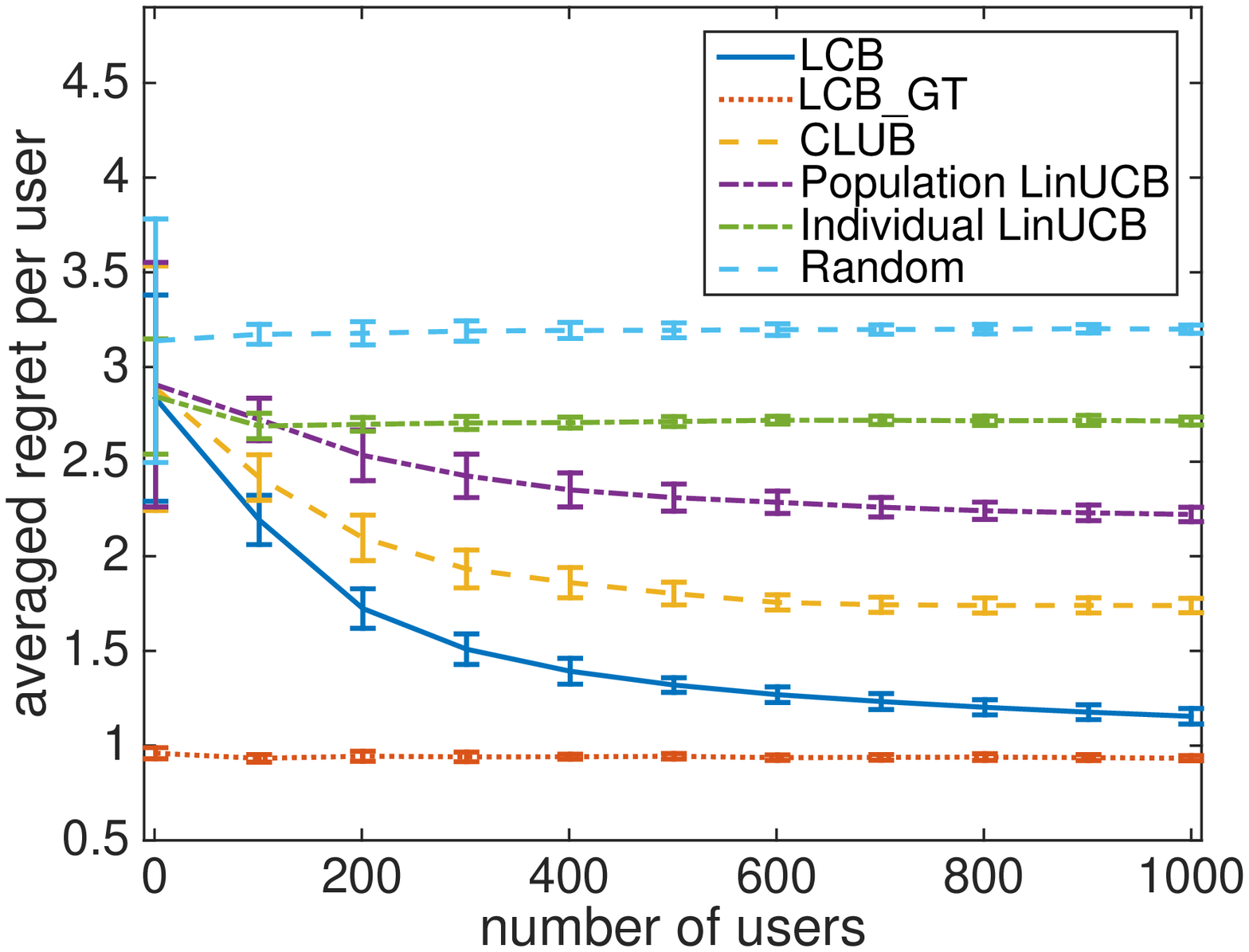}
  \caption{Averaged per-user regret vs.\ number of users.}
  \label{fig:sim_online}
\end{subfigure}
\begin{subfigure}{.495\textwidth}
  \centering
  \includegraphics[width=1\linewidth]{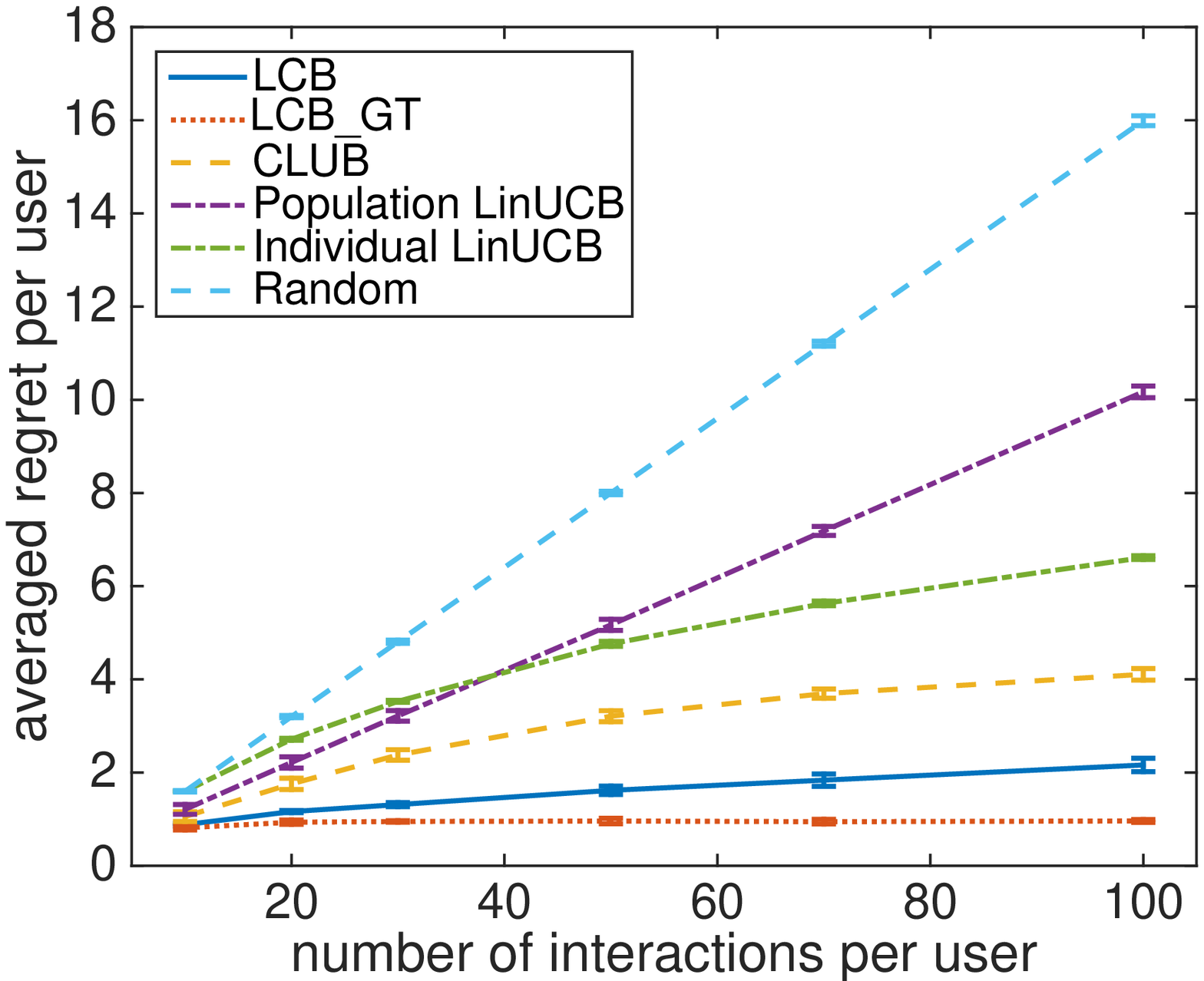}
  \caption{Averaged per-user regret vs.\ number of interactions per users.}
  \label{fig:sim_num_steps}
\end{subfigure}
\caption{Experiment results on simulation}
\label{fig:fig}
\end{minipage}
\end{figure}

\subsection{Experiments on Real World Dataset}
\label{exp:real_data}
\begin{figure*}[!ht]
  \centering
\begin{minipage}[c]{1\textwidth}
\begin{subfigure}{.33\textwidth}
  \centering
  \includegraphics[width=1\linewidth]{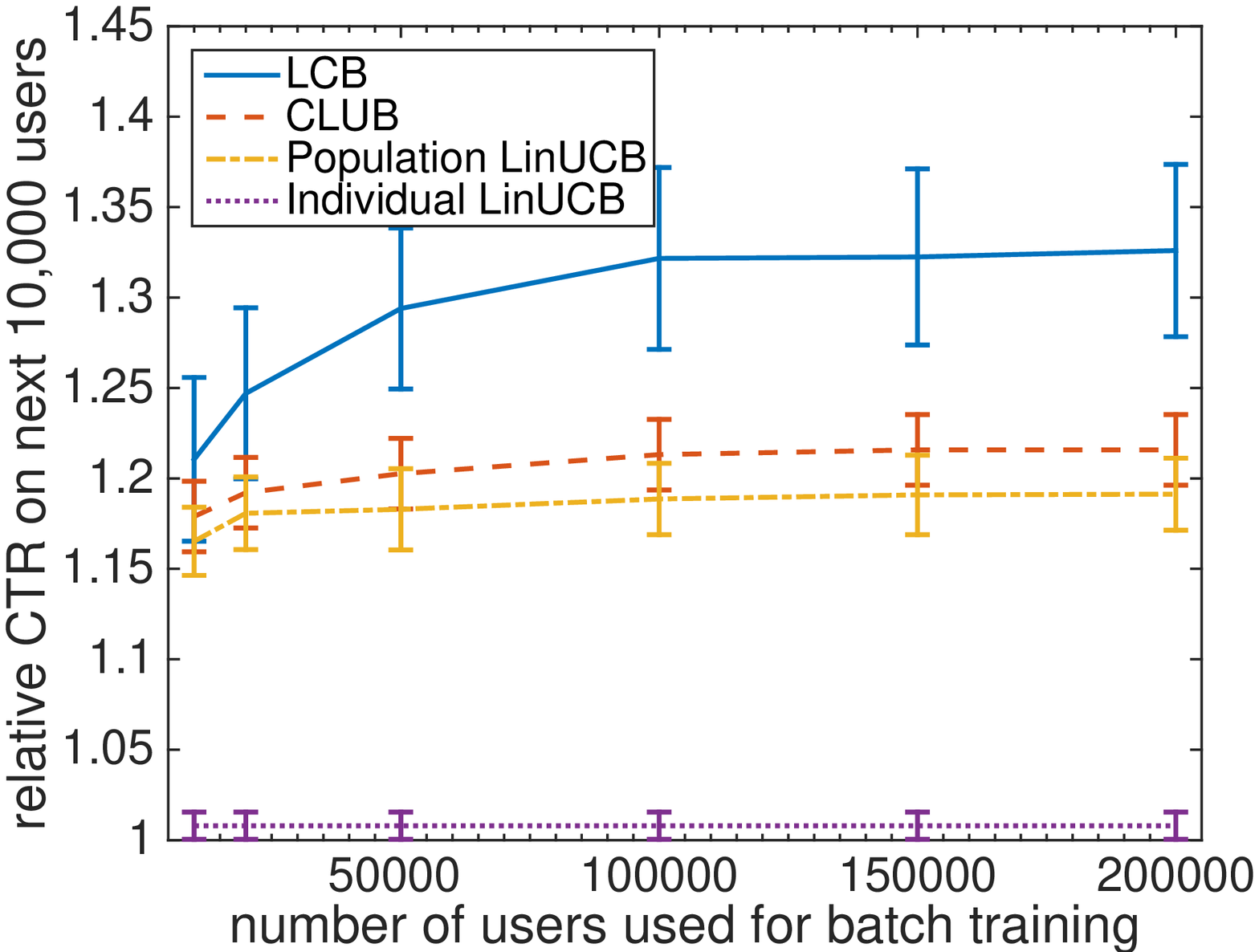}
  \caption{Relative CTR of 10,000 new users vs.\ number of users used for batch training.}
  \label{fig:yd_batch}
\end{subfigure}
\begin{subfigure}{.33\textwidth}
  \centering
  \includegraphics[width=1\linewidth]{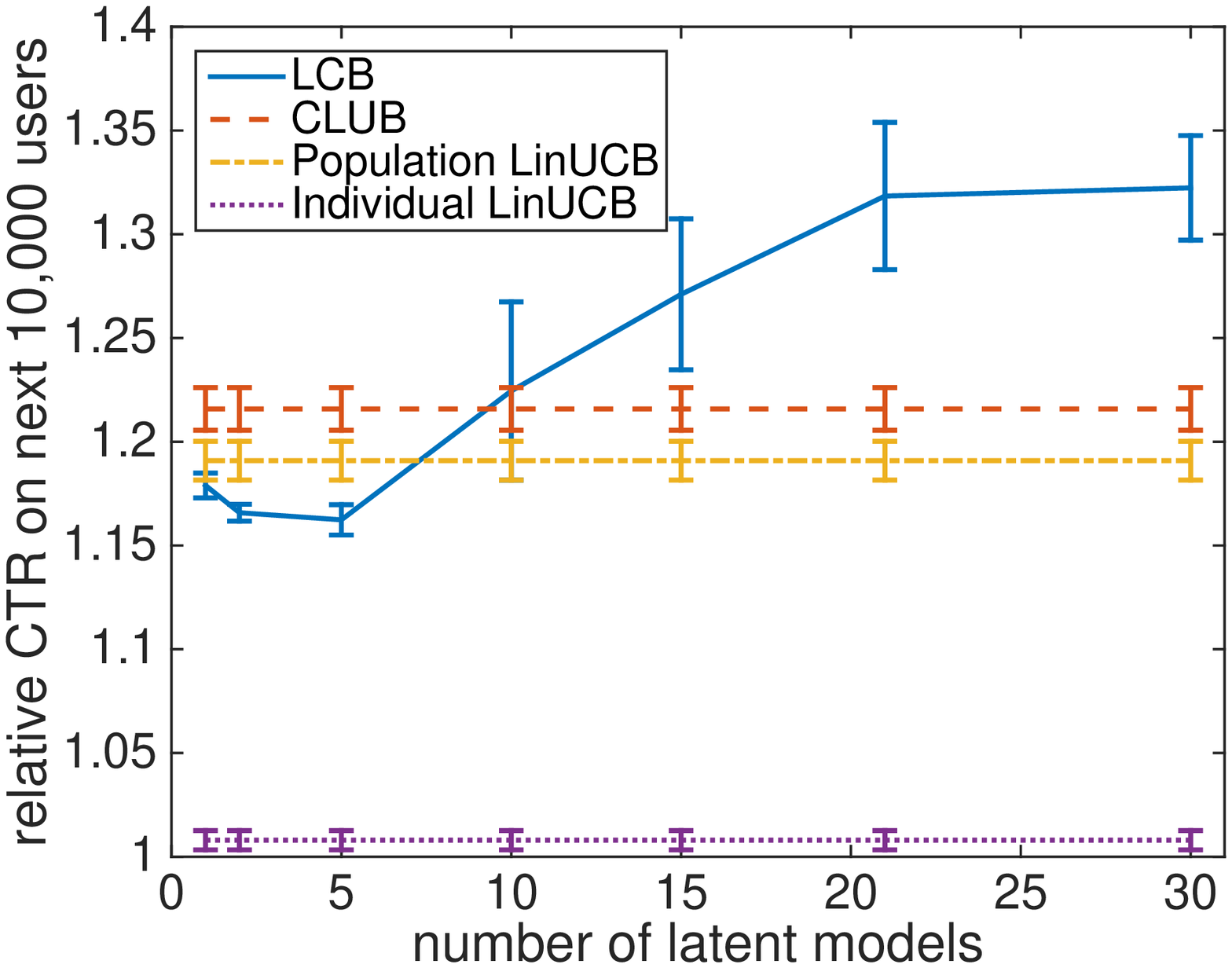}
  \caption{Relative CTR of 10,000 new users vs.\ number of latent models specified in batch training.}
  \label{fig:yd_num_models}
\end{subfigure}
\begin{subfigure}{.33\textwidth}
  \centering
  \includegraphics[width=1\linewidth]{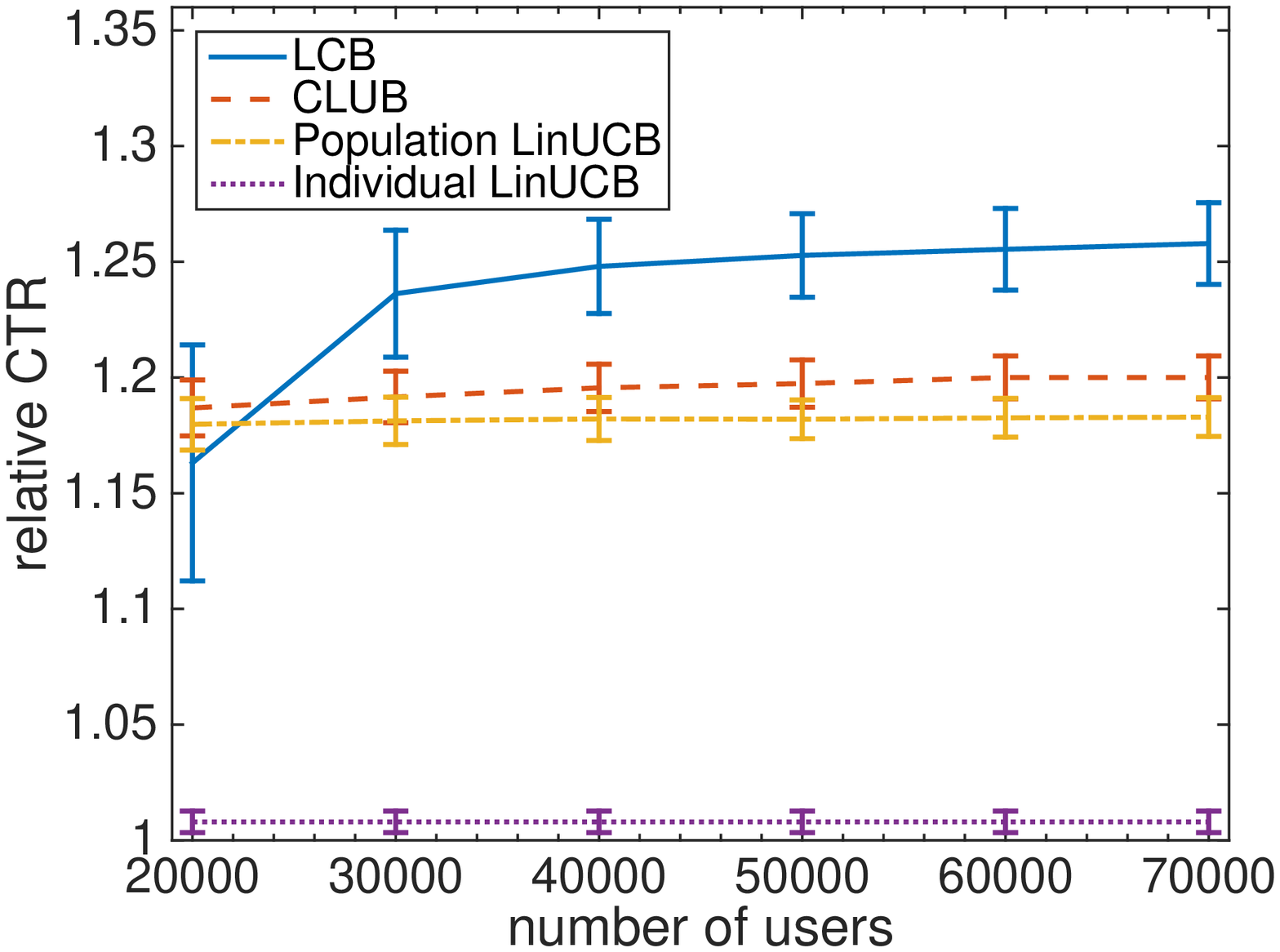}
  \caption{Relative CTR vs.\ number of users that have been interacted with the algorithm.}
  \label{fig:yd_online}
\end{subfigure}
\caption{Experiment results on a large real world news recommendation dataset.}
\label{fig:fig}
\end{minipage}
\end{figure*}
We evaluated our algorithm on a news feed dataset provided by Yahoo!. The
dataset contained $500,\!000$ users and all their visits in a one month
period. In each visit, a user was shown $25$ news articles from the top
down. User clicks (binary feedback) were logged. There were $21$ news
categories, and each news article belonged to $1\sim3$ categories. Therefore,
articles were represented as a 21-dimensional binary feature vector. User
features were not available because of privacy issues.

To the best of our knowledge, in our case there is no perfect solution for
\emph{unbiased} offline evaluation. For stationary algorithm one can use
propensity scoring~\cite{strehl2010learning}, however our algorithm is
nonstationary, and propensity score is not available in this dataset. One
state-of-the-art solution is rejection sampling based replay
method~\cite{li2011unbiased}. However, rejection sampling is quite sample
inefficient on our dataset because the policy which generated our dataset was
biased towards exploitation. Therefore, we adopted the queue
method~\cite{mandel2015queue}, a sample efficient offline evaluation method for
\emph{non-contextual} bandits, and extended it to our \emph{contextual} case.

To use the queue method, we defined arms as news categories instead of news
articles. As there were $21$ categories, we defined $21$ queues for each
user. The queues of each user were initialized with click labels ($0$ or $1$) of
articles shown to that user. For example, if an article belonged to two
categories, then its click label was added to the two corresponding
queues. Finally, we ran PCA to project the 21-dimensional article feature space
to a 6-dimensional lower space so that each category (arm) can be represented as
a dense vector.

We fixed the number of interactions per user ($T_u$) to $20$ for all users to
ensure all users were new users. We reported relative CTR, the algorithm's CTR
divided by the CTR in the data, due to confidentiality. In the first experiment,
we ran batch training: each algorithm was pre-trained with $20,\!000 -
200,\!000$ users, and then tested on the $10,\!000$ new users. Relative CTR of
the test users was reported. For LCB, we used the training users to learn $30$
latent models, then we directly ran phase 2 for the test users without
re-training latent models. Figure \ref{fig:yd_batch} shows the experiment
results. We can see that LCB achieved the highest CTR, and outperformed CLUB by
about $10\%$. Moreover, CLUB only outperformed Population LinUCB by about
$2\%$. One reason is that the rewards in our real dataset were binary and noisy,
so 20 samples per user ($T_u=20$) were not enough for CLUB to learn a good
regression model for each user and hence to learn a good latent graph
structure. In the second experiment, we varied the number of latent models of
LCB from $1$ to $30$. Similar to the batch training, each algorithm was
pre-trained with $150,\!000$ users and then tested on $10,\!000$ new users. The
result is shown in Figure \ref{fig:yd_num_models}. We can see that with 10 or
more models, LCB started to take the benefit of latent class structure and
outperformed CLUB and LinUCB. With $15$ and $30$ latent models, our approach
improved the CTR by about $5\%$ and $10\%$ respectively compared with CLUB and
Population LinUCB. The third experiment simulated the real world environment in
which users came sequentially and interacted with the algorithm. For LCB,
$20,\!000$ users were used in phase 1. In phase 2, $30$ latent models were
trained and re-trained after every $10,\!000$ users. To collect i.i.d data
points to better learn the latent models, we picked arms uniformly at random for
the first $5$ interactions of each user, and only used these data points to
train/re-train latent models. For all algorithms, we reported relative CTR after
every $10,\!000$ users. Results in Figure \ref{fig:yd_online} shows that LCB
achieved about $5\%$ higher CTR than CLUB and Population LinUCB.

\subsection{Pilot Results on User Study}
In this section, we show the pilot results of our user study with $10$ users,
$5$ for each algorithm. We compared two algorithms: Population LinUCB and
LCB. Since the experiments in Section \ref{exp:real_data} used the Yahoo!\ real
world dataset, so the learned models can be directly used for the user
study. For LCB, we used the learned latent models from Section
\ref{exp:real_data} and directly ran phase 2 of LCB. For Population LinUCB, we
used the learned Population LinUCB model from Section \ref{exp:real_data} to
initialize the LinUCB model used in the user study. Users interacted with the
algorithms through an app developed on the Android platform. $T_u=20$ for all
users. During each user interaction, the app requested $170$ latest news
articles from the Yahoo!\ news service in real time. Similar to Section
\ref{exp:real_data}, each news article was represented by a $21$-dimensional
vector. The algorithm then selected one of the articles for the user and
received user feedback (click). Table \ref{table:user_study} shows the CTR mean
and standard deviation achieved by these two algorithms. We can see from the
pilot results that LCB outperformed Population LinUCB. User study with more
users and algorithms is in progress.
\begin{table}[]
\small
\centering
\begin{tabular}{|c|c|}
\hline
Population LinUCB & LCB   \\ \hline
$0.196 \pm 0.096$  & $0.380 \pm 0.076$\\ \hline
\end{tabular}
\caption{CTR mean and standard deviation in user study.}
\label{table:user_study}
\end{table}
\section{Conclusion}
In this paper, we propose Latent Contextual Bandits, a contextual bandits
algorithm that learns the latent structure of users and leverages the learned
latent structure to make personalized recommendations for new users. We prove
both a problem-independent and a problem-dependent regret bound with respect to
the true policies of users. The regret bounds significantly improved over
baseline algorithms. We then demonstrate the benefit of our approach using both
simulation and an unbiased offline evaluation with a large real world dataset,
as well as a preliminary user study.

\section*{Acknowledgments}
This work was supported by the CMU-Yahoo!\ InMind project. We also gratefully
acknowledge the assistance and/or helpful feedback of Liangjie Hong, Suju Rajan,
Michal Valko, Saloni Potdar, Zhengyang Ruan, Linxi Zou, Mingzhi Zeng and the
pilot study participants.

\bibliographystyle{named}
\bibliography{lb_bib}
\end{document}